\documentclass{article}

\usepackage[utf8]{inputenc} 
\usepackage[T1]{fontenc}    
\usepackage{url}            
\usepackage{booktabs}       
\usepackage{amsfonts}       
\usepackage{nicefrac}       
\usepackage{microtype}      
\usepackage{xcolor}
\usepackage{authblk}
\usepackage{graphicx} 
\usepackage{graphbox}
\usepackage{amsmath} 
\usepackage{amsthm}
\usepackage{caption}
\usepackage{subcaption}
\usepackage{xparse}
\usepackage{listings}

\usepackage{grffile}

\usepackage[outdir=./]{epstopdf,epsfig}
\usepackage{xr}
\usepackage{hyperref}
\usepackage{cleveref}

\makeatletter
\newcommand*{\addFileDependency}[1]{
  \typeout{(#1)}
  \@addtofilelist{#1}
  \IfFileExists{#1}{}{\typeout{No file #1.}}
}
\makeatother

\usepackage[normalem]{ulem}
\usepackage{wrapfig}
\usepackage[export]{adjustbox}
\usepackage{geometry}
\usepackage{setspace}

\usepackage[
natbib=true,
style=numeric,
sorting=none,
url=false,
date=year,
urldate=short,
backend=biber,
language=american,
isbn=false,
doi=false,
eprint=false,
giveninits=true,
defernumbers=true,
maxnames=99]{biblatex}
\usepackage{lscape}
\usepackage{longtable}

\DeclareSourcemap{
  \maps[datatype=bibtex]{
    \map[overwrite]{
      \step[fieldsource=month, match=\regexp{\A(j|J)an(uary)?\Z}, replace=1]
      \step[fieldsource=month, match=\regexp{\A(f|F)eb(ruary)?\Z}, replace=2]
      \step[fieldsource=month, match=\regexp{\A(m|M)ar(ch)?\Z}, replace=3]
      \step[fieldsource=month, match=\regexp{\A(a|A)pr(il)?\Z}, replace=4]
      \step[fieldsource=month, match=\regexp{\A(m|M)ay\Z}, replace=5]
      \step[fieldsource=month, match=\regexp{\A(j|J)un(e)?\Z}, replace=6]
      \step[fieldsource=month, match=\regexp{\A(j|J)ul(y)?\Z}, replace=7]
      \step[fieldsource=month, match=\regexp{\A(a|A)ug(ust)?\Z}, replace=8]
      \step[fieldsource=month, match=\regexp{\A(s|S)ep(tember)?\Z}, replace=9]
      \step[fieldsource=month, match=\regexp{\A(o|O)ct(ober)?\Z}, replace=10]
      \step[fieldsource=month, match=\regexp{\A(n|N)ov(ember)?\Z}, replace=11]
      \step[fieldsource=month, match=\regexp{\A(d|D)ec(ember)?\Z}, replace=12]
    }
  }
}
\useunder{\uline}{\ul}{}

\NewDocumentCommand{\fullref}{sm}{%
  \IfBooleanTF{#1}{%
    \namecref{#2} \nameref*{#2}%
  }{%
    \namecref{#2} \nameref{#2}%
  }%
}

\theoremstyle{definition}
\newtheorem{definition}{Definition}[section]

\newtheorem*{lemma}{Lemma}

\newtheorem*{theorem}{Theorem}

\newtheorem*{corollary}{Corollary}

\graphicspath{{./figures/}{./}}

\title{Analysis of Neural Fragility: Bounding the Norm of a Rank-One Perturbation Matrix}

\author[1]{Adam Li}
\author[2]{Chester Huynh}

\affil[1]{Department of Computer Science, Columbia University, New York, United States}
\affil[2]{Microsoft, Seattle, United States}

\begin{document}

\maketitle

\footnotetext[1]{Also reachable at adam2392 at gmail dot com. Work done partly while at Department of Biomedical Engineering, Johns Hopkins University, Baltimore, MD, 21218}

\begin{abstract}
    Over 15 million epilepsy patients worldwide do not respond to drugs and require surgical treatment. Successful surgical treatment requires complete removal, or disconnection of the epileptogenic zone (EZ), but without a prospective biomarker of the EZ, surgical success rates vary between 30\%-70\%. Neural fragility is a model recently proposed to localize the EZ. Neural fragility is computed as the l2 norm of a structured rank-one perturbation of an estimated linear dynamical system. However, an analysis of its numerical properties have not been explored. We show that neural fragility is a well-defined model given a good estimator of the linear dynamical system from data. Specifically, we provide bounds on neural fragility as a function of the underlying linear system and noise.
\end{abstract}

\section{Introduction}
    Over 15 million epilepsy patients worldwide and 1 million in the US suffer from drug-resistant epilepsy (DRE) \cite{Brodie1997,Berg2006,Kwan2000,Berg2009,WHO2019a}. DRE patients have an increased risk of sudden death and are frequently hospitalized, burdened by epilepsy-related disabilities, and the cost of their care is a significant contributor to the \$16 billion dollars spent annually in the US treating epilepsy patients \cite{Ngugi2010,Begley2015b}. Approximately 50\% of DRE patients have focal DRE, where a specific brain region or regions, termed the epileptogenic zone (EZ), is necessary and sufficient for initiating seizures and whose removal (or disconnection) is necessary for complete abolition of seizures \cite{Luders,Jehi2018,PENFIELD1956,Penfield1954}. Successful surgical and neuromodulatory treatments can stop seizures altogether or allow them to be controlled with medications \cite{Morrell2006,Jobst2010a,Wiebe2001}, but outcomes for both treatments critically depend on accurate localization of the EZ. 
    
    Many methods have been proposed to localize the EZ \cite{Li2018,An2019,Weiss2019,Burns2014}, but neural fragility is a recent proposal that demonstrated impressive results. In \cite{Li2021NeuralZone,LiAdamInatiSaraZaghloulKareemSarma2017,Li2021}, "neural fragility" is introduced as a potential marker for the EZ. Neural fragility is computed as the norm of a rank-one structured perturbation matrix on a linear dynamical system. Neural fragility as a metric for localizing the EZ performs impressively when used as a feature for predicting surgical outcomes. Moreover, neural fragility presents an interpretable spatiotemporal heatmap that clinicians can view. Although the model performs impressively on a retrospective study of 91 subjects, an understanding of the statistical properties of neural fragility as a function of the system identification algorithm is lacking. 
    
    Neural fragility of EEG data relies on two steps: i) system identification of the underlying linear dynamical system and ii) computing the norm of the perturbation matrix applied to the estimated system. Many groups study the statistical properties of system identification. For example, least squares regression is used commonly for system identification of a linear system \cite{Li2017a,Brunton2016DiscoveringSystems,Li,}. Recently, \cite{Simchowitz2018LearningIdentification} identifies sharp bounds for learning linear systems and least-squares is an optimal learning algorithm for such systems. Although linear system identification from data is now well understood, the estimate of neural fragility from an estimated linear system has not been studied extensively yet.
    
    We here present an analysis of neural fragility to demonstrate lower and upper bounds on the value of neural fragility as a function of the underlying linear system and noise. We validate these bounds using simulations. Finally, we motivate further extensions of the neural fragility method based on applications to real data.
    
\section{Preliminaries}
    \subsection{Notation}
        We work with primarily vectors and matrices in real space. For the sake of space, all proofs are provided in \nameref{sec:supp_theory} in the online version of the paper. 
        
    \subsection{Continuous to Discrete-Time Neural Fragility}
        \label{subsec:continuous_to_discrete}
        In \cite{Sritharan2014}, brain activity is assumed to be a continuous time system. We assume brain recordings are governed by a discrete time-varying linear dynamical system, where any small window of time is linear time-invariant. To go from continous to discrete time, we sample at periodic intervals.
        
            $$\dot{x} = Ax$$
            
        where $x \in \mathbb{R}^d$ is the vector of d-dimensional EEG activity and $A \in M_{d}$ is a d-by-d matrix governing the linear dynamics. If we sample discretely at uniform periodic intervals, as we do in real EEG data, we get:
        
            $$\frac{x(t_{n+1}) - x(t_n)}{t_{n+1} - t_n} = Ax(t_n)$$
            
        from discretization of the linear system. Then re-arranging terms, we get:
        
        \begin{align}
            x(t_{n+1}) = x(t_n) + Ax(t_n) (t_{n+1} - t_n) \\
                = x(t_n) + Ax(t_n) \Delta t \quad \text{(Define $\Delta t = (t_{n+1} - t_n)$)}\\
                = (I + A \Delta t)x(t_n)
        \end{align}
        
        Now, for stable linear systems, we get the following $\lambda_i (I + A \Delta t) = 1 + \lambda_i (A) \Delta t$, and then we get the following criterion for discrete-time stability:
        
            $$Re(\lambda_i (I + A \Delta t)) = 1 + Re(\lambda_i(A))\Delta t < 1$$
            $$Im(\lambda_i (I + A \Delta t)) = Im(\lambda_i(A))\Delta t$$
            
        Now, it can be rewritten as:
        
            $$(1 + Re(\lambda_i(A)) \Delta t)^2 + (Im(\lambda_i(A)) \Delta t)^2 < 1$$
        
        so the marginal stability depends on the sampling rate, $\frac{1}{\Delta t}$ being sufficiently high, which in many EEG recordings are (i.e. 1000 Hz or higher).
        
    \subsection{Prior Results}
        Here, we summarize useful prior results for our analysis of neural fragility.
        
        First, we restate the theorem of \cite{Sritharan2014,LiAdamInatiSaraZaghloulKareemSarma2017}, which derives how to compute neural fragility given a linear dynamical system, represented by the matrix, A. We say $\lambda \in \sigma(A)$ is an eigenvalue in the spectrum of A and has a corresponding eigenvector, $v \in \mathbb{R}^n$ such that: $Av = \lambda v$. 
        
        \begin{theorem} [Computation of neural fragility from linear system]
            Suppose $A \in M_n$ represents the state matrix of a linear dynamical system. Assume that $r \in \mathbb{C}$ is a number (possibly complex) that is not an eigenvalue of A. Then, for all $k = 1,...,n$, there exists a rank-one matrix, $\Delta \in R_k(\Gamma)$, such that: 
            
                $$r \in \sigma(A + \Delta)$$
            
            and with minimum 2-norm. Moreover, $\Delta$ can be solved analytically by the equation:
            
            \begin{equation}
            \label{perturbation}
                \mathbf{\widehat{\Delta}} = \mathbf{B}^T (\mathbf{B}\mathbf{B}^T)^{-1} \mathbf{b}] \mathbf{e_{\widehat{k}}}^T
            \end{equation}
            
            where 
            
            \begin{align}
                \mathbf{B}(r, k) &= \begin{bmatrix} 
                Im\{e_k^T(\mathbf{A} - r\mathbf{I})^{-T}\} \\ 
                Re\{e_k^T(\mathbf{A} - r\mathbf{I})^{-T}\} \end{bmatrix} \\
                \mathbf{b} &=\begin{bmatrix} 0 \\ -1\end{bmatrix}
            \end{align}
            
            $k$ is the index at which the perturbation is computed, $e_k \in \mathbb{R}^n$ is a unit vector with the one at the kth position. 
            
            Moreover, when $r \in \mathbb{R}$, then:
            
            $$\Gamma = -\frac{(rI - A)^{-1}e_k}{e_k^T(rI - A)^{-T}(rI - A)^{-1}e_k}$$
            
            which is the $n \times 1$ vector that perturbs the kth row of A.
        \end{theorem}
        
        This theorem differs slightly from \cite{Sritharan2014} because we use a discrete time model, but the proof follows as in \cite{Sritharan2014}. Next, we restate a few key results that will be useful for proving various bounds in the next section. 
        
        We remind the readers of what is known as the Neumann Series, which generalizes the geometric series of real numbers. 
        
        \begin{definition} [Neumann Series]
            A Neumann series of a matrix, T is an infinite series:
            
                $$\sum_{k=0}^\infty T^k$$
        \end{definition}
        
        We have the following theorem that utilizes the definition of the Neumann Series.
        
        \begin{lemma}
            For any matrix, $A \in M_n(\mathbf{C})$, with $||A|| < 1$. The matrix, $(I - A)$ is invertible and
            
                $$||(I - A)^{-1}|| \le \frac{1}{1 - ||A||}$$
        \end{lemma}
        \begin{proof}
            We use the matrix version of the Taylor series to expand $(I - A)^{-1}$ for $||A|| < 1$, such that we get the convergent series:
            
                $$(I - A)^{-1} = I + A + A^2 + A^3 + ...$$
                
            Thus, taking the norm of both sides:
            
            \begin{align}
                ||(I - A)^{-1}|| = ||I + A + A^2 + ...|| \\
                    & \le ||I|| + ||A|| + ||A^2|| + ... \quad \text{(Sub-additivity of norms)} \\
                    & = \frac{1}{1 - ||A||} \quad \text{(Geometric series for ||A|| < 1)}
            \end{align}
        \end{proof}
        
        Using this lemma, one has the following bound on the norm of the resolvent. 
        
        \begin{lemma}
            For any $A \in M_n(\mathbb{C})$ and $z \in \mathbb{C}$, such that $|z| > ||A||$, then the resolvent $Res(z)$ exists and 
            
                $$||Res(z)|| \le \frac{1}{|z| - ||A||}$$
        \end{lemma}
        \begin{proof}
            Since $|z| > ||A||$, then $||\frac{A}{z}|| < 1$, so we can apply the previous lemma on the quantity $\frac{A}{z}$.
            
            $$(I - \frac{A}{z})^{-1} = z(zI - A)^{-1} = (I + A/z + A^2/z^2 + ...)$$
            
            such that:
            
            $$(zI - A)^{-1} = z^{-1}(I + A/z + A^2/z^2 + ...)$$
            
            We can take the norm on both sides of this equation and utilize the previous lemma to obtain:
            
            $$||(zI - A)^{-1}|| \le \frac{1}{|z|} \frac{1}{1 - ||A/z||} = \frac{1}{|z| - ||A||}$$
        \end{proof}
        
        Next we define the notion of relative boundedness with respect to a linear operator. 
        
        \begin{definition}
            Let A and T be matrices with the same domain space, but not necessarily the same range space. Then for $a, b$ non-negative constants, if 
                
                $$||Au|| \le a ||u|| + b ||Tu||$$
                
            Then we say A is relatively bounded with respect to T, or A is T-bounded.
        \end{definition}
        
        In \cite{Kato1995PerturbationOperators}, Theorem 1.16 (page 196) states the stability of bounded invertibility, which we will leverage later. It states the following:
        
        \begin{theorem} [Stability of bounded invertibility from \cite{Kato1995PerturbationOperators}]
            Let A and T be linear operators from $\mathbb{R}^{n} \rightarrow \mathbb{R}^n$ (i.e. $n \times n$ matrices). Assume that $A^{-1}$ exists and is T-bounded with the constants $a, b$ satisfying the following inequality:
            
                $$a ||T^{1}|| + b < 1$$
            
            Then we have the following result: $S = T+A$ is invertible and:
            
                $$||S^{1}|| \le \frac{||T^{-1}||}{1 - a||T^{-1}|| - b}$$
                
            and
            
                $$||S^{1}|| \le \frac{||T^{-1}||}{1 - a||T^{-1}|| - b}$$
        \end{theorem}
        \begin{corollary} [Stability of bounded invertibility for bounded linear operators]
            \label{corr:stability_bounded_inv}
            If A is bounded, and we assume that T is A-bounded with constants a = ||T|| and b = 0, S = T + A, and $||A|| < 1 / ||T^{-1}||$, then we have:
            
                $$||S^{-1}|| \le \frac{||T^{-1}||}{1 - ||A||||T^{-1}||}$$
                
            and
            
                $$||S^{-1} - T^{-1}|| \le \frac{||A|| ||T^{-1}||^2}{1 - ||A||||T^{-1}||}$$
        \end{corollary}
        
\section{Results}
    We perform a theoretical analysis of neural fragility to demonstrate that its values are bounded mainly as a function of the properties of the underlying system, extending work in \cite{LiAdamInatiSaraZaghloulKareemSarma2017,Sritharan2014,Li2021NeuralZone}.
    
    \subsection{Neural fragility is a well-defined metric}
        We show that neural fragility is a well-defined metric in the sense that it reflects the true fragility of the system, given that we have an optimal estimator for the linear system over any time window. Neural fragility is defined by the norm of a perturbation vector, $\Gamma \in \mathbb{R}^d$, applied to a linear system, represented by $A$. Since we have to estimate $A$ with $\hat{A}$ from iEEG data, we would like $||\Gamma(\hat{A})|| \approx ||\Gamma(A)||$, where $\Gamma(A)$ is the perturbation vector computed given system, A, and $\Gamma(\hat{A})$ is the perturbation vector computed given the estimated system, $\hat{A}$. The following lemma and theorem encapsulate this.
        
        \begin{lemma}
            Assume, we are given $A \in M_n(\mathbb{R})$ with $||A|| < 1$. Then, we have that:
            
            $$||\Gamma_A|| \le \frac{||(A - rI)^{-1}||}{1 - ||(A - rI)^{-1}||} (|r| + 1)^2$$
        \end{lemma}
        
        \begin{theorem}
            Assume, we are given $A \in M_n(\mathbb{R})$ with $||A|| < 1$, and $E \in M_n(\mathbb{R})$, such that $||E|| < \epsilon < ||A||$. We define $\hat{A} := A + E$. $\Gamma_{\hat{A}}$ is the perturbation vector obtained by solving for neural fragility on $\hat{A}$. Then, we have that:
                
                $$||\Gamma_{\hat{A}}|| \le \frac{||(A - rI)^{-1}||}{1 - ||E||\ ||(A - rI)^{-1}||} (|r| + 1 + \epsilon)^2$$
        \end{theorem}
        
        This lemma and theorem informs us that if we can obtain an estimate of $A$, with $E: = \hat{A} - A$ with "small" norm, then $||\Gamma_{\hat{A}}||$ will be bounded by terms solely mainly on properties of the linear system and the perturbation radius, $r$.
        
    \subsection{Theoretical Analysis of Neural Fragility Model}
    \label{sec:supp_theory}
    
        Here, we summarize the theoretical analysis of neural fragility. As some preliminaries, we first review notation. We say that $A \in M_n$ is a $n \times n$ matrix; we only consider real matrices in this work. We denote, $M_n^1$ as the space of $n \times n$ matrices that have rank of one. Then we say $C_k(\Gamma)$ is the space of matrices with all zeros except for one column, with $\Gamma \in \mathbb{R}^n$ occupying the kth row. Then $R_k(\Gamma)$ is the space of matrices with all zeros except for one row, with $\Gamma$ occupying the kth row. We say that $Res(z)$ is resolvent matrix parametrized by $z \mapsto (A - zI)^{-1}$ for a given A matrix. It is defined for $z \notin \sigma(A)$.
    
    \subsubsection{Bounds on neural fragility - the norm of the perturbation matrix}
        In this section, we prove a bound on neural fragility, $||\Gamma||_2$. These bounds are derived from the fact that the computation of $\Gamma$ is a function of the resolvent of A. Thus, our primary strategy is to link resolvent bounds to our problem.
        
        First, because of the unique structure of the problem, we remind our readers of some facts. We define $\Delta_k = \begin{pmatrix} 0 & \hdots & \Gamma & \hdots & 0 \end{pmatrix}$, where the matrix is all zeros except for the kth column.
        
        The operator norm of $\Delta$ is equivalent to the Frobenius norm and also the 2-2 matrix norm.
        
        \begin{align}
            ||\Delta||_op = \max_{||v|| = 1} ||\Delta v|| = \max_{||v||=1} ||\Gamma e_k^T v|| = \max_{||v||=1} ||\Gamma|| |e_k^T v| = ||\Gamma|| ||e_k = ||\Gamma|| \\
            ||\Delta||_F = \sqrt{tr(\Delta \Delta^T)} = \sqrt{tr(\Gamma e_k^T e_k \Gamma^T)} = \sqrt{(\Gamma^T \Gamma)(e_k^T e_k)} = ||\Gamma|| ||e_k|| = ||\Gamma|| \\
            ||\Delta||_{2,2} = ||\Gamma|| \quad \text{(by definition)}
        \end{align}
        
        Note, that because the operator and Frobenius norm are unitarily invariant, then this holds if $\Delta$ was defined as a row perturbation matrix as well. Thus the following results hold regardless of which "norm" we choose. When $||\Gamma||$ is computed in practice, we use the l2 vector-norm.
        
        \paragraph{Assuming we have the true linear system, A}
            The first is assuming we have the true A matrix that characterizes the system. We remind readers of the Bauer-Fike theorem, which states:
            
            \begin{theorem} [Bauer-Fike Theorem from \cite{BAUERNormsTheorems.}]
                Let $r$ be an eigenvalue of $A + \Gamma$. We assume A is diagonalizable. Then there exists $\lambda \in \sigma(A)$ such that:
                
                $$|\lambda - r|  \le \kappa_p(V) ||\Gamma||_p$$
            
                where $\kappa_p(X)$ is the condition number in p-norm and $V \in M_n(\mathbb{C})$ is the eigenvector matrix of A, such that: $A = V \Lambda V^{-1}$, and $\Lambda$ is the diagonal matrix of eigenvalues.
            \end{theorem}
            
            Another way of stating the theorem is:
            
                $$\frac{|\lambda - r|}{\kappa_p(V)}  \le  ||\Gamma||_p$$
            
            This informs us that given the the desired radius of perturbation, r, then a well conditioned A matrix will result in a non-trivial lower-bound for the norm of $\Gamma$.
            
            \begin{corollary} [Naive bound of neural fragility]
                \label{corr:naive_lb_neural_frag}
                Assume we are given $A \in M_n(\mathbb{R})$ that is diagonalizable with eigenvalues $\lambda_1 \ge \lambda_2 \ge \hdots \lambda_n$. We compute neural fragility of A, by perturbing to $r \in \mathbb{R}$, such that $r > ||A||$. Then  
                
                $$||\Gamma||_2 \ge \frac{|r - \lambda_n|}{\kappa_p(V)}$$
                
            \end{corollary}
            \begin{proof}
                This is a consequence of Bauer-Fike theorem.
            \end{proof}
            
            This theorem tells us that for increasing radius of perturbation, the norms of $\Gamma$ will get uniformly larger and larger. However, if the original linear system has a high condition number, then the lower-bound is very small. Next, we obtain a form for the bound on $||\Gamma_A||$.
            
            \begin{theorem}[Upper bound of neural fragility]
                Assume, we are given $A \in M_n(\mathbb{R})$ with $||A|| < 1$. Then, we have that:
                
                $$||\Gamma_A|| \le \frac{||(A - rI)^{-1}||}{1 - ||(A - rI)^{-1}||} (|r| + 1)^2$$
            \end{theorem}
            
            This lemma provides an upper bound on the value of neural fragility, $||\Gamma_A||$, based on the true linear system, $A$. Combined with Corollary \ref{corr:naive_lb_neural_frag} we can obtain a range of neural fragility values we expect to see, based on the radius of perturbation (r), and properties of the linear system (A).
            
        \paragraph{Assuming we estimate the linear system, A}
            The next result assumes that we do not have the true A matrix, but rather a noisy version of it, $\hat{A} = A + E$. In this section, we abuse notation a bit and for every norm in this section, we mean $||.||$, we mean $||.||_{2,2}$, the 2-2 entry-wise matrix norm of a matrix.
        
            \begin{theorem} [Upper bound on neural fragility on estimated linear system]
                Assume, we are given $A \in M_n(\mathbb{R})$ with $||A|| < 1$, and $E \in M_n(\mathbb{R})$, such that $||E|| < \epsilon < ||A||$. We define $\hat{A} := A + E$. Then, we have that:
                
                $$||\Gamma_{\hat{A}}|| \le \frac{||(A - rI)^{-1}||}{1 - ||E||\ ||(A - rI)^{-1}||} (|r| + 1 + \epsilon)^2$$
            \end{theorem}

            This lemma informs us that if we have a noisy version of our linear system, $\hat{A} = A + E$, then as long as the norm of the perturbation matrix, E, is small, then we will obtain roughly a similar upper bound as $||\Gamma_A||$. 
        
        
            
        
        
                
                
                
                
                

    
    
\section{Discussion}
    
    In this paper, we introduce neural fragility as a function of an estimated linear dynamical system from data. We analyze some of its properties as a result of linear matrix theory and determine simple lower and upper bounds. Understanding how to better estimate linear systems, A, to form better estimates of neural fragility, $\Gamma$ is an interesting line of future research.

\section{Author Contributions}
    
    AL conceived the project. AL and CH wrote the theoretical results. AL and CH wrote the paper with input from the other authors. 

\printbibliography

@article{Wiebe2001,
    title = {{A randomized, controlled trial of surgery for temporal-lobe epilepsy}},
    year = {2001},
    journal = {New England Journal of Medicine},
    author = {Wiebe, Smauel and Blume, Warren T. and Girvin, John P. and Eliasziw, Michael},
    number = {5},
    month = {8},
    pages = {311--318},
    volume = {345},
    url = {http://www.nejm.org/doi/abs/10.1056/NEJM200108023450501},
    doi = {10.1056/NEJM200108023450501},
    issn = {00284793},
    pmid = {11484687}
}

@misc{Morrell2006,
    title = {{Brain stimulation for epilepsy: Can scheduled or responsive neurostimulation stop seizures?}},
    year = {2006},
    booktitle = {Current Opinion in Neurology},
    author = {Morrell, Martha},
    number = {2},
    month = {4},
    pages = {164--168},
    volume = {19},
    url = {https://insights.ovid.com/crossref?an=00019052-200604000-00011},
    doi = {10.1097/01.wco.0000218233.60217.84},
    issn = {13507540}
}

@article{Jobst2010a,
    title = {{Brain stimulation for the treatment of epilepsy: Brain Stimulation in Epilepsy}},
    year = {2010},
    journal = {Epilepsia},
    author = {Jobst, Barbara C. and Darcey, Terrance M. and Thadani, Vijay M. and Roberts, David W.},
    month = {7},
    pages = {88--92},
    volume = {51},
    url = {http://www.ncbi.nlm.nih.gov/pubmed/20618409 http://doi.wiley.com/10.1111/j.1528-1167.2010.02618.x},
    doi = {10.1111/j.1528-1167.2010.02618.x},
    issn = {00139580, 15281167},
    pmid = {20618409}
}

@misc{Brodie1997,
    title = {{Commission on European Affairs: Appropriate standards of epilepsy care across Europe}},
    year = {1997},
    booktitle = {Epilepsia},
    author = {Brodie, M. J. and Shorvon, S. D. and Canger, R. and Halasz, P. and Johannessen, S. and Thompson, P. and Wieser, H. G. and Wolf, P.},
    number = {11},
    month = {11},
    pages = {1245--1250},
    volume = {38},
    publisher = {Blackwell Publishing Ltd},
    url = {http://doi.wiley.com/10.1111/j.1528-1157.1997.tb01224.x},
    doi = {10.1111/j.1528-1157.1997.tb01224.x},
    issn = {00139580},
    pmid = {9579928}
}

@article{Berg2006,
    title = {{Defining intractability: Comparisons among published definitions}},
    year = {2006},
    journal = {Epilepsia},
    author = {Berg, Anne T. and Kelly, Molly M.},
    number = {2},
    month = {2},
    pages = {431--436},
    volume = {47},
    publisher = {Blackwell Publishing Inc},
    url = {http://doi.wiley.com/10.1111/j.1528-1167.2006.00440.x},
    doi = {10.1111/j.1528-1167.2006.00440.x},
    issn = {00139580},
    pmid = {16499772}
}

@article{Kwan2000,
    title = {{Early Identification of Refractory Epilepsy}},
    year = {2000},
    journal = {New England Journal of Medicine},
    author = {Kwan, Patrick and Brodie, Martin J.},
    number = {5},
    month = {2},
    pages = {314--319},
    volume = {342},
    publisher = {Massachusetts Medical Society},
    url = {http://www.nejm.org/doi/abs/10.1056/NEJM200002033420503},
    doi = {10.1056/NEJM200002033420503},
    issn = {0028-4793},
    pmid = {10660394}
}

@misc{WHO2019a,
    title = {{Epilepsy}},
    year = {2019},
    author = {{WHO} and Organization, World Health},
    url = {https://www.who.int/news-room/fact-sheets/detail/epilepsy}
}

@article{Penfield1954,
    title = {{Epilepsy and the Functional Anatomy of the Human Brain.}},
    year = {1954},
    journal = {Journal of the American Medical Association},
    author = {Penfield, Wilder and Jasper, Herbert},
    edition = {[1st ed.].},
    number = {1},
    pages = {86},
    volume = {155},
    publisher = {Little Brown},
    url = {http://jama.jamanetwork.com/article.aspx?doi=10.1001/jama.1954.03690190092039},
    address = {Boston},
    doi = {10.1001/jama.1954.03690190092039},
    issn = {0002-9955},
    pmid = {6976708}
}

@article{PENFIELD1956,
    title = {{Epileptogenic lesions.}},
    year = {1956},
    journal = {Acta neurologica et psychiatrica Belgica},
    author = {Penfield, W},
    number = {2},
    month = {2},
    pages = {75--88},
    volume = {56},
    url = {http://www.ncbi.nlm.nih.gov/pubmed/13313098},
    issn = {00016284},
    pmid = {13313098}
}

@article{Ngugi2010,
    title = {{Estimation of the burden of active and life-time epilepsy: A meta-analytic approach}},
    year = {2010},
    journal = {Epilepsia},
    author = {Ngugi, Anthony K. and Bottomley, Christian and Kleinschmidt, Immo and Sander, Josemir W. and Newton, Charles R.},
    number = {5},
    month = {5},
    pages = {883--890},
    volume = {51},
    url = {http://www.ncbi.nlm.nih.gov/pubmed/20067507 http://www.pubmedcentral.nih.gov/articlerender.fcgi?artid=PMC3410521 http://doi.wiley.com/10.1111/j.1528-1167.2009.02481.x},
    doi = {10.1111/j.1528-1167.2009.02481.x},
    issn = {00139580},
    pmid = {20067507}
}

@article{Sritharan2014,
    title = {{Fragility in Dynamic Networks: Application to Neural Networks in the Epileptic Cortex}},
    year = {2014},
    journal = {Neural Computation},
    author = {Sritharan, Duluxan and Sarma, Sridevi V.},
    number = {10},
    month = {10},
    pages = {2294--2327},
    volume = {26},
    publisher = { MIT Press  One Rogers St., Cambridge, MA 02142-1209 USA journals-info@mit.edu  },
    url = {http://www.mitpressjournals.org/doi/10.1162/NECO_a_00644},
    doi = {10.1162/NECO{\_}a{\_}00644},
    issn = {0899-7667}
}

@inproceedings{LiAdamInatiSaraZaghloulKareemSarma2017,
    title = {{Fragility in Epileptic Networks : the Epileptogenic Zone}},
    year = {2017},
    booktitle = {American Control Conference},
    author = {Li, Adam and Inati, Sara and Zaghloul, Kareem and Sarma, Sridevi},
    pages = {1--8},
    isbn = {978-1-5090-5992-8},
    doi = {10.23919/ACC.2017.7963378},
    issn = {07431619}
}

@article{Berg2009,
    title = {{Identification of Pharmacoresistant Epilepsy}},
    year = {2009},
    journal = {Neurologic Clinics},
    author = {Berg, Anne T.},
    number = {4},
    month = {11},
    pages = {1003--1013},
    volume = {27},
    publisher = {NIH Public Access},
    url = {http://www.ncbi.nlm.nih.gov/pubmed/19853220 http://www.pubmedcentral.nih.gov/articlerender.fcgi?artid=PMC2827183},
    doi = {10.1016/j.ncl.2009.06.001},
    issn = {07338619},
    pmid = {19853220},
    arxivId = {NIHMS150003}
}

@inproceedings{Li2017a,
    title = {{Linear time-varying model characterizes invasive EEG signals generated from complex epileptic networks}},
    year = {2017},
    booktitle = {Proceedings of the Annual International Conference of the IEEE Engineering in Medicine and Biology Society, EMBS},
    author = {Li, Adam and Gunnarsdottir, Kristin M. and Inati, Sara and Zaghloul, Kareem and Gale, John and Bulacio, Juan and Martinez-Gonzalez, Jorge and Sarma, Sridevi V.},
    number = {1},
    month = {7},
    pages = {2802--2805},
    publisher = {IEEE},
    url = {https://ieeexplore.ieee.org/document/8037439/},
    isbn = {9781509028092},
    doi = {10.1109/EMBC.2017.8037439},
    issn = {1557170X},
    pmid = {29060480}
}

@article{Weiss2019,
    title = {{Localizing epileptogenic regions using high-frequency oscillations and machine learning}},
    year = {2019},
    journal = {Biomarkers in Medicine},
    author = {Weiss, Shennan A and Waldman, Zachary and Raimondo, Federico and Slezak, Diego and Donmez, Mustafa and Worrell, Gregory and Bragin, Anatol and Engel, Jerome and Staba, Richard and Sperling, Michael},
    number = {5},
    month = {4},
    pages = {409--418},
    volume = {13},
    url = {http://www.ncbi.nlm.nih.gov/pubmed/31044598 https://www.futuremedicine.com/doi/10.2217/bmm-2018-0335},
    doi = {10.2217/bmm-2018-0335},
    issn = {1752-0363},
    pmid = {31044598}
}

@article{Burns2014,
    title = {{Network dynamics of the brain and influence of the epileptic seizure onset zone}},
    year = {2014},
    journal = {Proceedings of the National Academy of Sciences},
    author = {Burns, Samuel P. and Santaniello, Sabato and Yaffe, Robert B. and Jouny, Christophe C. and Crone, Nathan E. and Bergey, Gregory K. and Anderson, William S. and Sarma, Sridevi V.},
    number = {49},
    month = {12},
    pages = {E5321-E5330},
    volume = {111},
    publisher = {National Academy of Sciences},
    url = {http://www.pnas.org/lookup/doi/10.1073/pnas.1401752111 www.pnas.org/cgi/doi/10.1073/pnas.1401752111PNAS%7C},
    doi = {10.1073/pnas.1401752111},
    issn = {0027-8424},
    pmid = {25404339}
}

@article{An2019,
    title = {{Optimization of surgical intervention outside the epileptogenic zone in the Virtual Epileptic Patient (VEP)}},
    year = {2019},
    journal = {PLoS computational biology},
    author = {An, Sora and Bartolomei, Fabrice and Guye, Maxime and Jirsa, Viktor},
    editor = {Marinazzo, Daniele},
    number = {6},
    month = {6},
    pages = {e1007051},
    volume = {15},
    publisher = {Public Library of Science},
    url = {http://dx.plos.org/10.1371/journal.pcbi.1007051},
    doi = {10.1371/journal.pcbi.1007051},
    issn = {15537358}
}

@misc{Begley2015b,
    title = {{The direct cost of epilepsy in the United States: A systematic review of estimates}},
    year = {2015},
    booktitle = {Epilepsia},
    author = {Begley, Charles E. and Durgin, Tracy L.},
    number = {9},
    month = {9},
    pages = {1376--1387},
    volume = {56},
    url = {http://www.ncbi.nlm.nih.gov/pubmed/26216617 http://doi.wiley.com/10.1111/epi.13084},
    doi = {10.1111/epi.13084},
    issn = {15281167},
    pmid = {26216617}
}

@misc{Jehi2018,
    title = {{The epileptogenic zone: Concept and definition}},
    year = {2018},
    booktitle = {Epilepsy Currents},
    author = {Jehi, Lara},
    number = {1},
    pages = {12--16},
    volume = {18},
    publisher = {American Epilepsy Society},
    url = {http://www.ncbi.nlm.nih.gov/pubmed/29844752 http://www.pubmedcentral.nih.gov/articlerender.fcgi?artid=PMC5963498},
    doi = {10.5698/1535-7597.18.1.12},
    issn = {15357511},
    pmid = {29844752}
}

@article{Luders,
    title = {{The epileptogenic zone: General principles}},
    year = {2006},
    journal = {Epileptic Disorders},
    author = {L{\"{u}}ders, Hans O. and Najm, Imad and Nair, Dileep and Widdess-Walsh, Peter and Bingman, William},
    number = {SUPPL. 2},
    volume = {8},
    url = {http://campus.neurochirurgie.fr/IMG/pdf/version PDF the epileptic zone.pdf},
    isbn = {978 1 84184 576 0},
    issn = {12949361},
    pmid = {17012067}
}

@article{Li2018,
    title = {{Using network analysis to localize the epileptogenic zone from invasive EEG recordings in intractable focal epilepsy}},
    year = {2018},
    journal = {Network Neuroscience},
    author = {Li, Adam and Chennuri, Bhaskar and Subramanian, Sandya and Yaffe, Robert and Gliske, Steve and Stacey, William and Norton, Robert and Jordan, Austin and Zaghloul, Kareem A. and Inati, Sara K. and Agrawal, Shubhi and Haagensen, Jennifer J. and Hopp, Jennifer and Atallah, Chalita and Johnson, Emily and Crone, Nathan and Anderson, William S. and Fitzgerald, Zach and Bulacio, Juan and Gale, John T. and Sarma, Sridevi V. and Gonzalez-Martinez, Jorge},
    number = {02},
    month = {6},
    pages = {218--240},
    volume = {02},
    publisher = { MIT Press  One Rogers Street, Cambridge, MA 02142-1209 USA journals-info@mit.edu  },
    url = {https://www.mitpressjournals.org/doi/abs/10.1162/netn_a_00043},
    doi = {10.1162/netn{\_}a{\_}00043},
    issn = {2472-1751}
}

@inproceedings{Li,
    title = {{Virtual Cortical Stimulation Mapping of Epilepsy Networks to Localize the Epileptogenic Zone}},
    year = {2019},
    booktitle = {Proceedings of the Annual International Conference of the IEEE Engineering in Medicine and Biology Society, EMBS},
    author = {Li, Adam and Fitzgerald, Zachary and Hopp, Jennifer and Johnson, Emily and Crone, Nathan and Bulacio, Juan and Martinez-Gonzalez, Jorge and Inati, Sara and Zaghloul, Kareem and Sarma, Sridevi V.},
    number = {Cc},
    month = {10},
    pages = {2328--2331},
    volume = {2019},
    publisher = {Institute of Electrical and Electronics Engineers (IEEE)},
    url = {https://pubmed.ncbi.nlm.nih.gov/31946366/},
    isbn = {9781538613115},
    doi = {10.1109/EMBC.2019.8856591},
    issn = {1557170X},
    pmid = {31946366}
}

@article{Brunton2016DiscoveringSystems,
    title = {{Discovering governing equations from data by sparse identification of nonlinear dynamical systems}},
    year = {2016},
    journal = {Proceedings of the National Academy of Sciences of the United States of America},
    author = {Brunton, Steven L. and Proctor, Joshua L. and Kutz, J. Nathan and Bialek, William},
    number = {15},
    month = {4},
    pages = {3932--3937},
    volume = {113},
    publisher = {National Academy of Sciences},
    url = {https://www.pnas.org/content/113/15/3932 https://www.pnas.org/content/113/15/3932.abstract},
    doi = {10.1073/PNAS.1517384113/-/DCSUPPLEMENTAL},
    issn = {10916490},
    arxivId = {1509.03580}
}

@article{Simchowitz2018LearningIdentification,
    title = {{Learning Without Mixing: Towards A Sharp Analysis of Linear System Identification}},
    year = {2018},
    journal = {arXiv},
    author = {Simchowitz, Max and Mania, Horia and Tu, Stephen and Jordan, Michael I. and Recht, Benjamin},
    month = {2},
    publisher = {arXiv},
    url = {http://arxiv.org/abs/1802.08334},
    arxivId = {1802.08334}
}

@article{Li2021NeuralZone,
    title = {{Neural fragility as an EEG marker of the seizure onset zone}},
    year = {2021},
    journal = {Nature Neuroscience 2021 24:10},
    author = {Li, Adam and Huynh, Chester and Fitzgerald, Zachary and Cajigas, Iahn and Brusko, Damian and Jagid, Jonathan and Claudio, Angel O. and Kanner, Andres M. and Hopp, Jennifer and Chen, Stephanie and Haagensen, Jennifer and Johnson, Emily and Anderson, William and Crone, Nathan and Inati, Sara and Zaghloul, Kareem A. and Bulacio, Juan and Gonzalez-Martinez, Jorge and Sarma, Sridevi V.},
    number = {10},
    month = {8},
    pages = {1465--1474},
    volume = {24},
    publisher = {Nature Publishing Group},
    url = {https://www.nature.com/articles/s41593-021-00901-w},
    doi = {10.1038/s41593-021-00901-w},
    issn = {1546-1726},
    pmid = {34354282}
}

@article{BAUERNormsTheorems.,
    title = {{Norms and exclusion theorems.}},
    journal = {Numerische Mathematik},
    author = {BAUER, F.L. and FIKE, C.T.},
    pages = {137--141},
    volume = {2},
    issn = {0029-599X}
}

@article{Kato1995PerturbationOperators,
    title = {{Perturbation Theory for Linear Operators}},
    year = {1995},
    author = {Kato, Tosio},
    series = {Classics in Mathematics},
    volume = {132},
    publisher = {Springer Berlin Heidelberg},
    url = {http://link.springer.com/10.1007/978-3-642-66282-9},
    address = {Berlin, Heidelberg},
    isbn = {978-3-540-58661-6},
    doi = {10.1007/978-3-642-66282-9}
}


\appendix

\section{Appendix}

\subsection{Proofs}
    Here, we include the proofs of the three theorems regarding the bounds on neural fragility.

    \begin{theorem}[Upper-bound on neural fragility]
        Assume, we are given $A \in M_n(\mathbb{R})$ with $||A|| < 1$. Then, we have that:
        
        $$||\Gamma_A|| \le \frac{||(A - rI)^{-1}||}{1 - ||(A - rI)^{-1}||} (|r| + 1)^2$$
    \end{theorem}
    \begin{proof}
        We first, write out the form of $\Gamma_A(r, k)$:
        
        $$\Gamma_A = \frac{Res(r; A) e_k}{e_k^T Res(r; A)^T Res(r; A) e_k}$$
        
        Then taking the norm on both sides:
        
        \begin{align}
            ||\Gamma_A|| = ||\frac{Res(r; A) e_k}{e_k^T Res(r; A)^T Res(r; A) e_k}|| \nonumber \\
            \le \frac{||Res(r; A)||}{||Res(r; A)^T Res(r; A)||} \nonumber \\
            = \frac{||(A - rI)^{-1}||}{||Res(r; A)^T Res(r; A)||} \label{eq:initial_ub_trueA}
        \end{align}
        
        Recall that:
        
            $$1 = ||I|| = ||A A^{-1}|| \le ||A|| ||A^{-1}||$$
        
        for an induced matrix norm.
            
        Next, by lower-bounding $||Res(r; A)^T Res(r;A)||$, we can further upper-bound the above quantity. Using this fact, and properties of matrix norms, we obtain the following lower-bound on this quantity:
        
        \begin{align}
            ||Res(r; A)^T Res(r;\hat{A})|| = ||(rI - A)^{-T} (rI - A)^{-1}||  \nonumber \\
            = ||((rI - A)(rI - A)^T)^{-1}|| \nonumber \\
            \ge \frac{1}{||(rI - A)(rI - A)||} \nonumber \\
            \ge \frac{1}{r^2 ||I|| + |r| ||A|| + |r| ||A^T|| + ||A|| ||A^T||} \nonumber \\
            = \frac{1}{(|r| + ||A||)^2} \nonumber \\
            \ge \frac{1}{(|r| + 1)^2} \label{eq:intermed_ub_trueA}
        \end{align}
        
        We leverage the stability of bounded invertibility in \cite{Kato1995PerturbationOperators} (Stability theorems; pg 196). Finally, we combine the results \ref{eq:initial_ub_trueA}, \ref{eq:intermed_ub_trueA} and the Corollary for \nameref{corr:stability_bounded_inv} to obtain:
        
        \begin{align}
            ||\Gamma_A|| \le \frac{||(A - rI)^{-1}||}{||Res(r; A)^T Res(r; A)||} \\
            \le \frac{||(A - rI)^{-1}||}{\frac{1}{(|r| + 1)^2}} \\
            = ||(A - rI)^{-1}|| (|r| + 1)^2 \\
            \le \frac{||(A - rI)^{-1}||}{1 - ||(A - rI)^{-1}||} (|r| + 1)^2 \\
            = \frac{||(A - rI)^{-1}||}{1 - ||(A - rI)^{-1}||} (|r| + 1)^2
        \end{align}
    \end{proof}

    \begin{theorem} [Upper bound on neural fragility on estimated linear system]
        Assume, we are given $A \in M_n(\mathbb{R})$ with $||A|| < 1$, and $E \in M_n(\mathbb{R})$, such that $||E|| < \epsilon < ||A||$. We define $\hat{A} := A + E$. Then, we have that:
        
        $$||\Gamma_{\hat{A}}|| \le \frac{||(A - rI)^{-1}||}{1 - ||E||\ ||(A - rI)^{-1}||} (|r| + 1 + \epsilon)^2$$
    \end{theorem}
    \begin{proof}
        The proof proceeds in a similar fashion to the case when we have the true A matrix. We first, write out the form of $\Gamma_{\hat{A}}(r, k)$:
        
        $$\Gamma_{\hat{A}} = \frac{Res(r; A+E) e_k}{e_k^T Res(r; A+E)^T Res(r; A+E) e_k}$$
        
        Then taking the norm on both sides:
        
        \begin{align}
            ||\Gamma_{\hat{A}}|| = ||\frac{Res(r; A+E) e_k}{e_k^T Res(r; A+E)^T Res(r; A+E) e_k}|| \nonumber \\
            \le \frac{||Res(r; A+E)||}{||Res(r; A+E)^T Res(r; A+E)||} \nonumber \\
            = \frac{||(A + E - rI)^{-1}||}{||Res(r; A+E)^T Res(r; A+E)||} \label{eq:initial_ub}
        \end{align}
        
        Recall that:
        
            $$1 = ||I|| = ||A A^{-1}|| \le ||A|| ||A^{-1}||$$
            
        Next, by lower-bounding $||Res(r; \hat{A})^T Res(r;\hat{A})||$, we can further upper-bound the above quantity. Using this fact, and properties of matrix norms, we obtain the following lower-bound on this quantity:
        
        \begin{align}
            ||Res(r; \hat{A})^T Res(r;\hat{A})|| = ||(rI - \hat{A})^{-T} (rI - \hat{A})^{-1}||  \nonumber \\
            = ||((rI - \hat{A})(rI - \hat{A})^T)^{-1}|| \nonumber \\
            \ge \frac{1}{||(rI - \hat{A})(rI - \hat{A})||} \nonumber \\
            \ge \frac{1}{r^2 ||I|| + |r| ||\hat{A}|| + |r| ||\hat{A}^T|| + ||\hat{A}|| ||\hat{A}^T||} \nonumber \\
            = \frac{1}{(|r| + ||\hat{A}||)^2} \nonumber \\
            \ge \frac{1}{(|r| + ||A|| + ||E||)^2} \nonumber \\
            \ge \frac{1}{(|r| + 1 + \epsilon)^2} \label{eq:intermed_ub}
        \end{align}
        
        We leverage the stability of bounded invertibility in \cite{Kato1995PerturbationOperators} (Stability theorems; pg 196). Finally, we combine the results \ref{eq:initial_ub}, \ref{eq:intermed_ub} and the Corollary for \nameref{corr:stability_bounded_inv} to obtain:
        
        \begin{align}
            ||\Gamma_{\hat{A}}|| \le \frac{||(A + E - rI)^{-1}||}{||Res(r; A+E)^T Res(r; A+E)||} \\
            \le \frac{||(A + E - rI)^{-1}||}{\frac{1}{(|r| + 1 + \epsilon)^2}} \\
            = ||(A + E - rI)^{-1}|| (|r| + 1 + \epsilon)^2 \\
            \le \frac{||(A - rI)^{-1}||}{1 - ||E|| ||(A - rI)^{-1}||} (|r| + 1 + \epsilon)^2 \\
            = \frac{||(A - rI)^{-1}||}{1 - ||E||\ ||(A - rI)^{-1}||} (|r| + 1 + \epsilon)^2
        \end{align}
    \end{proof}

\end{document}